\documentclass{article} 
\usepackage{times,authblk}

\usepackage[margin=1in]{geometry}
\usepackage{mkolar_definitions}
\usepackage{natbib,rotating, color,subfigure}

\def \supp {\textrm{supp}}

\begin{document}

\title{Detecting Activations over Graphs using \\Spanning Tree Wavelet Bases}

\author[1,3]{
James Sharpnack
\thanks{jsharpna@cs.cmu.edu}}
\author[2]{
Akshay Krishnamurthy
\thanks{akshaykr@cs.cmu.edu}}
\author[1]{
Aarti Singh
\thanks{aarti@cs.cmu.edu}
}
\affil[1]{Machine Learning Department\\
Carnegie Mellon University}
\affil[2]{Computer Science Department\\
Carnegie Mellon University}
\affil[3]{Statistics Department\\
Carnegie Mellon University}

%

\maketitle

\begin{abstract}

We consider the detection of activations over graphs under Gaussian noise, where signals are piece-wise constant over the graph.
Despite the wide applicability of such a detection algorithm, there has been little success in the development of computationally feasible methods with prove-able theoretical guarantees for general graph topologies.
We cast this as a hypothesis testing problem, and first provide a universal necessary condition for asymptotic distinguishability of the null and alternative hypotheses.
We then introduce the spanning tree wavelet basis over graphs, a localized basis that
reflects the topology of the graph, and prove that for any spanning tree, this
approach can distinguish null from alternative in a low signal-to-noise regime.
Lastly, we improve on this result and show that using the uniform spanning tree
in the basis construction yields a randomized test with stronger theoretical
guarantees that in many cases matches our necessary conditions. 
Specifically, we obtain near-optimal performance in edge transitive graphs,
$k$-nearest neighbor graphs, and $\epsilon$-graphs.


\end{abstract}

\section{Introduction}
This paper focuses on the problem of detecting activations over a graph when observations are corrupted by noise. 
The problem of detecting graph-structured activations is relevant to many applications including identifying congestion in router and road networks, eliciting preferences in social networks, and detecting viruses in human and computer networks.
Furthermore, these applications require that the method is scalable to large graphs.
Luckily, computer science boasts a plethora of efficient graph based algorithms that we can adapt to the detection framework.

\subsection{Contributions}

In this paper, we will be testing if there is a non-zero piece-wise constant activation pattern on the graph given observations that are corrupted by Gaussian white noise.
We show that correctly distinguishing the null and alternative hypotheses is impossible if the signal-to-noise ratio does not grow quickly with respect to the allowable number of discontinuities in the activation pattern (Section 2). 
Since a test based on the scan statistic which matches the observations with all possible activation patterns by brute force is infeasible, we propose a Haar wavelet basis construction for general graphs, which is formed by hierarchically dividing a spanning tree of the graph (Section 3).
We find that the size and power of the test can be bounded in terms of the number of signal discontinuities and the spanning tree, immediately giving us a result for any spanning tree.
We then propose choosing a spanning tree uniformly at random (this can be done
efficiently), and show that this bound can be improved by a factor of the
average effective resistance of the edges across which the signal is non-constant (Section 4).
With this machinery in place we are able to show that for edge transitive
graphs, such as lattices, $k$-nearest neighbor graphs, and $\epsilon$ geometric random
graphs, our test is nearly-optimal in that the upper bounds match the
fundamental limits of detection up to logarithm factors (Section 5).

\subsection{Problem Setup}

Consider an undirected graph $G$ defined by a set of vertices $V$ ($|V| = n$) and undirected edges $E$ ($|E| = m$) which are unordered pairs of vertices.
Throughout this study we will assume that the graph $G$ is known.
The statistical setting that we will address is the normal means model,
\[
\yb = \xb + \epsilonb
\]
where $\xb, \in \RR^V$, $\epsilonb \sim N(0,\sigma^2\Ib_V)$, and $\sigma^2$ is known. 
Specifically, we assume that there are parameters $\rho, \mu$ (possibly dependent on $n$) such that 
\[
\Xcal = \{\xb \in \RR^V : |\{ (v,w) \in E : x_v \ne x_w \}| \le \rho, \|\xb\| \ge \mu\}
\] 
defines the class of possible $\xb$.
Hence, the possible signals have few edges across which the values of $\xb$ differ.
In graph-structured activation detection we are concerned with statistically testing the null and alternative hypotheses,
\begin{eqnarray}
\begin{aligned}
H_0:&\ \yb \sim N(\zero,\sigma^2 \Ib) \\
H_1:&\ \yb \sim N(\xb,\sigma^2 \Ib), \xb \in \Xcal\\
\end{aligned}
\label{eq:main_problem}
\end{eqnarray}
$H_0$ represents business as usual while $H_1$ encompasses all of the foreseeable anomalous activity.
Let a test be a mapping $T(\yb) \in \{0,1\}$, where $1$ indicates that we reject the null.

It is imperative that we control both the probability of false alarm, and the false acceptance of the null.
To this end, we define our measure of risk to be
\[
R(T) = \EE_{\zero} [T] + \sup_{\xb \in \Xcal} \EE_\xb [1 - T]
\]
where $\EE_\xb$ denote the expectation with respect to $\yb \sim N(\xb,\sigma^2 \Ib)$.
The test $T$ may be randomized, in which case the risk is $\EE_T R(T)$.
Notice that if the distribution of the random test $T$ is independent of $\xb$, then $\EE_T \sup_{\xb \in \Xcal} \EE_\xb [1 - T] = \sup_{\xb \in \Xcal} \EE_{T,\xb} [1 - T]$.
This is the setting of \cite{arias2011detection} which we should contrast to the Bayesian setup in \cite{addario2010combinatorial}.
We will say that $H_0$ and $H_1$ are {\em asymptotically distinguished} by a test, $T$, if $\lim_{n \rightarrow \infty} R(T) = 0$.
If such a test exists then $H_0$ and $H_1$ are asymptotically distinguished, otherwise they are asymptotically indistinguishable.   

To aid us in our study we introduce some mathematical terminology.
Let the edge-incidence matrix of $G$ be $\nabla \in \RR^{E \times V}$ such that for $(v,w) \in E$, $\nabla_{(v,w),v} = 1$, $\nabla_{(v,w),w} = -1$ (the order of $(v,w)$ is chosen arbitrarily) and is $0$ elsewhere.
For a vector, $\wb \in \RR^E$, $\supp(\wb) = \{v \in V : \wb \ne 0\}$ and $\|
\wb \|_0 = |\supp(\wb)|$, so $\|\nabla \xb\|_0 \le \rho$ for all $\xb \in \Xcal$.
We will be constructing spanning trees $\Tcal$ of the graph $G$, which are connected subsets of $E$ with no cycles.
Furthermore, we will denote the edge-incidence matrix of $\Tcal$ as $\nabla_\Tcal$.

\subsection{Related Work}

The statistical problem that we are addressing can be broadly classified as a high-dimensional Gaussian goodness-of-fit test.
This is a well studied problem when the structure of $H_1$ derives from a smooth function space such as an ellipsoid, Besov space or Sobolev space \cite{ingster1987minimax,ingster2003nonparametric}.
The function space $\Xcal$ that we are proposing is combinatorial in nature.
This statistical problem has only recently been studied theoretically \cite{addario2010combinatorial,arias2011detection}, although to the best of our knowledge none have addressed the problem under arbitrary graph structure.
More broadly, this work falls under the purview of multiple hypothesis testing, which has a rich history \cite{benjamini1995controlling}.
Unfortunately, aside from a few special cases \cite{hall2010innovated}, the multiple tests are assumed to be independent, making any such work not applicable to our setting.


In this paper, we evaluate our method by it's ability to distinguish $H_0$ from $H_1$, however
the procedure is based on constructing a wavelet basis over graphs which is 
relevant for other problems such as denoising and compression. 
Wavelets are multi-resolution bases that can represent inhomogeneous signals efficiently
using a few non-zero wavelet coefficients which makes them attractive for denoising, compression and detection.
As a result, they have been used extensively in mathematics, signal processing, statistics and physics \cite{mallat1999wavelet}.
They have also been used with great success in statistics, with extensive theoretical guarantees \cite{donoho1995adapting,hardle1998wavelets,vidakovic1999statistical}.
Recently there has been some attention paid to developing wavelets for graphs.
Unfortunately, most of these have either focused on graphs with a known hierarchical structure \cite{gavishICML10,ram2011generalized,singh2010detecting}, or do not come with approximation or sparsifying properties that can be used for our class of graph functions $\Xcal$ \cite{hammond2011wavelets,diff_wav}.

\section{Universal Lower Bound}

In order to more completely understand the problem of detecting anomalous activity in graphs, we prove that there is a universal minimum signal strength under which $H_0$ and $H_1$ are asymptotically indistinguishable.
The proof is based on a lemma developed in \cite{arias2008searching}, but the strategic use of this lemma is novel.
Our construction of the `worst case' prior gives a significantly tighter bound than would a more naive implementation. 
Indeed, it is interesting to note that the worst case prior is a uniform distribution of the largest unstructured signals that we are allowed in $H_1$ that are nearly disjoint.

\begin{theorem}
  \label{thm:lower_bound}
  Let the maximum degree of $G$ be $d_{\max}$.
  Consider the alternative, $H_1$, in which the cut size of each signal in $\Xcal$ is bounded by $\rho$, with $\lim_{n \rightarrow \infty} \rho = \infty$ and $\rho \le \bar d n$.  $H_0$ and $H_1$ are asymptotically indistinguishable if 
\[
\frac{\mu}{\sigma} = o \left( \sqrt{\min \{\frac{\rho}{d_{\max}},\sqrt{n}\}} \right)
\]
\end{theorem}

\begin{proof}

We begin by constructing a prior distribution over $\Xcal$.
This portion of the proof derives from the analysis in \cite{arias2008searching} and closely mirrors that of \cite{addario2010combinatorial,arias2011detection}.
We will suppose that we have some subset $\Scal \subseteq 2^V$ such that we will draw an $S \in \Scal$ uniformly at random.
Then the signal is constructed $X = \frac{\mu}{\sqrt{|S|}} \one_S$ giving us a prior distribution $\pi$ over $\Xcal$.  Call the Bayes risk $R^*$.

\begin{lemma}
  \label{lem:ery}\cite{addario2010combinatorial}
  Let $S$ and $S'$ be drawn uniformly at random from $\Scal$.
  Then the Bayes risk $R^*$ is bounded by
  \[
  R^* \ge 1 - \frac{1}{2} \sqrt{\EE \exp \left( \frac{\mu^2}{2 \sigma^2} \frac{|S \cap S'|}{\sqrt{|S||S'|}} \right) - 1}
  \]
\end{lemma}

Hence, if $\EE \exp \left( \frac{\mu^2}{2 \sigma^2} \frac{|S \cap S'|}{\sqrt{|S||S'|}} \right) \rightarrow 1$, then $H_0$ and $H_1$ are asymptotically indistinguishable.
Let $p = \lfloor \min \{\rho /d_{\max},\sqrt{n}\} \rfloor$ and construct $\Scal$ to be all subsets of $V$ of size $p$.
Then,
\[
\EE \exp{ \frac{\mu^2}{2 \sigma^2} \frac{|S \cap S'|}{\sqrt{|S||S'|}}} = 
\EE \exp{ \frac{\mu^2}{2 p \sigma^2} |S \cap S'| }
\]

Let $\{z_i\}_{i = 1}^p$ be Bernoulli trials with success probability $p/n$.
We see that the distribution of $|S \cap S'|$ is invariant under conditioning on $S'$ and then it is equivalent to sampling without replacement from a population in which there are $p$ successes.
By Theorem 4 in \cite{hoeffding1963probability} we know that for $t > 0$, $\EE e^{t |S \cap S'|} \le \EE e^{t \sum_{i = 1}^p z_i}$.
Let $t = \frac{\mu^2}{2 p \sigma^2}$, by the generating function of Bernoulli random variables,
\[
\EE e^{t \sum_{i = 1}^p z_i} = \left( 1 + \frac pn \left( e^{\frac{\mu^2}{2 p \sigma^2}} - 1 \right)\right)^p
\]
By the assumption $\frac{\mu^2}{\sigma^2} = o(p)$ so for any $c > 0$ for $n$ large enough 
\[
\left( 1 + \frac pn \left( e^{\frac{\mu^2}{2 p \sigma^2}} - 1 \right)\right)^p \le 
\left( 1 + c \frac p n \right)^p \le \left( 1 + \frac c p \right)^p \rightarrow e^c
\]
because $p \le \sqrt{n}$.
Hence, $\EE e^{t |S \cap S'|} \rightarrow 1$. 
All that remains is to notice that the cut sizes of $S \in \Scal$ are bounded by $\rho$ because the cut sizes are bounded by $p d_{\max} \le \rho$.
\end{proof}

\section{Spanning Tree Wavelets}
\label{sec:wavelets}
In this section, we present an algorithm for constructing a wavelet basis given
a spanning tree and we characterize its performance for the detection
problem~\eqref{eq:main_problem}.

Informally, we would like to construct a basis $\Bb$ for which each edge $e \in
\mathcal{T}$ is activated by very few basis elements, where we say that an edge
$e$ is activated by element ${\bf b}$ if $e \in supp(\nabla_{\mathcal{T}}{\bf
  b})$. As we will show, upper bounding the number of basis elements that
activate any edge will be essential in analyzing the performance of our
estimator $||{\bf By}||_{\infty}$.

We construct our wavelet basis $\Bb$ recursively, by first finding a seed vertex
in the spanning tree such that the subtrees adjacent to the seed have at most
$\lceil n/2 \rceil$ vertices and then by including basis elements localized on
these subtrees in $\Bb$. We recurse on each subtree, adding higher-resolution
elements to our basis, and consequently constructing a complete wavelet
basis. The first phase of the algorithm ensures that the depth of the recursion
is at most $\lceil \log n\rceil$ and the second ensures that each edge is
activated by at most $\lceil \log d \rceil$ basis elements per recursive
call. Combining these two shows that each edge is activated by at most $\lceil
\log d \rceil \lceil \log n \rceil$ basis elements.


Finding a balancing vertex in the tree parallels the technique
in~\cite{pearl1986structuring}, which finds a balancing edge. The algorithm
starts from any vertex $v \in \mathcal{T}$ and moves along $\mathcal{T}$ to a
neighboring vertex $w$ that lies in the largest connected component of
$\mathcal{T} \setminus v$. The algorithm repeats this process (moving from $v$
to $w$) until the largest connected component of $\mathcal{T} \setminus w$ is
larger than the largest connected component of $\mathcal{T} \setminus v$ at
which point it returns $v$. We call this the {\em FindBalance} algorithm.


Once we have a balancing vertex $v$, we form wavelets that are constant over the
connected components of $\Tcal \backslash v$ such that any vertex is supported
by at most $\log d$ wavelets. Let $d_v$ be the degree of the balancing vertex
$v$ and let $c_1, \ldots c_{d_v}$ be the connected components of $\Tcal
\backslash v$ (with $v$ added to the smallest component). Our algorithm acts as
if $c_1, \ldots c_{d_v}$ form a chain structure and constructs the Haar wavelet
basis over them. We call this algorithm {\em FormWavelets}:

\begin{enumerate}
\item Let $C_1 = \cup_{i \le d_v/2} c_i$ and $C_2 = \cup_{i > d_v/2}$
\item Form the following basis element and add it to $\Bb$:
\[
\bb = \frac{\sqrt{|C_1||C_2|}}{\sqrt{|C_1| + |C_2|}} \left[ \frac{1}{|C_1|} \mathbf{1}_{C_1} - \frac{1}{|C_2|} \mathbf{1}_{C_2}  \right]
\]
\item Recurse at (1) with the subcomponents of $C_1$ and $C_2$ with partitions $\{ c_i \}_{i \le p/2}$ and $\{ c_i \}_{i > p/2}$ respectively.
\end{enumerate}

Our algorithm recursively constructs basis elements using the {\em FindBalance}
and {\em FormWavelets} routines on subtrees of $\mathcal{T}$. We initialize
$\Tcal$ to be a spanning tree of the graph and start with no elements in our
basis.
\begin{enumerate}
\item Let $v$ be the output of {\em FindBalance} applied to $\Tcal$.
\item Let $\Tcal_1,..,\Tcal_{d_v}$ be the connected components of $\Tcal
  \backslash v$ and add $v$ to the smallest component.
\item Add the basis elements constructed in {\em FormWavelets} when applied to $\Tcal_1,...,\Tcal_{d_v}$
\item For each $i \in [d_v]$, recursively apply (1) - (4) on $\Tcal_i$ as long
  as $|\Tcal_i| > 2$.
\end{enumerate}

As we will see, controlling the sparsity, $||{\bf Bx}||_0$ is essential in
analyzing the performance of the estimator $||{\bf By}||_{\infty}$. The main
theoretical guarantee of our basis construction algorithm is that signals with
small cuts in $G$ are sparse in $\Bb$. Specifically, we prove the following key
lemma in the appendix:

\begin{lemma}
\label{lem:tree_cut_bound}
Let $\nabla$ be the incidence matrix of $G$ and $\nabla_\Tcal$ be the incidence
matrix of $\Tcal$ (where $\Tcal$ has degree at most $d$).  Then $\norm{\nabla
  \xb}_0$ is the cut size of pattern $\xb \in \RR^{V(G)}$.  Then for any $\xb
\in \RR^{V(G)}$,
\begin{eqnarray}
\norm{\Bb \xb}_0 \le \norm{\nabla_\Tcal \xb}_0 \lceil \log d \rceil \lceil \log n \rceil \le \norm{\nabla \xb}_0
\lceil \log d \rceil \lceil \log n \rceil
\end{eqnarray}
\end{lemma}

Equipped with Lemma~\ref{lem:tree_cut_bound} we can now characterize the
performance of the estimator $||{\bf By}||_{\infty}$ on any signal ${\bf
  x}$. Our bound depends on the choice of spanning tree $\mathcal{T}$,
specifically via the quantity $||\nabla_{\mathcal{T}}{\bf x}||_0$, the cut size
of ${\bf x}$ in $\mathcal{T}$.
The proof of the following can be found in the appendix.

\begin{theorem}
\label{thm:detection_bd}
Perform the test in which we reject the null if $\norm{\Bb \yb}_\infty > \tau$.  
Set $\tau = \sigma \sqrt{2 \log (n/\delta)}$.  If
\begin{eqnarray}
\frac{\mu}{\sigma} \ge \sqrt{2 \norm{\nabla_\Tcal \xb}_0 \lceil \log d \rceil \lceil \log n \rceil} (\sqrt{\log (1 / \delta)} + \sqrt{\log(n/\delta)})
\end{eqnarray}
then under $H_0$, $\PP \{ \textrm{Reject} \} \le \delta$, and under $H_1$, $\PP \{ \textrm{Reject} \} \ge 1 - \delta$.
\end{theorem}

\begin{remark}
\label{rem:simple_bound}
For any tree we have $\norm{\nabla_\Tcal{\bf x}}_0 \le \norm{\nabla
  {\bf x}}_0$ for all patterns ${\bf x}$, so that for the sparse cut alternative
we can have both Type I and Type II errors $\le \delta$ as long as:
\begin{eqnarray}
\frac{\mu}{\sigma} \ge \sqrt{2 \rho \lceil \log d \rceil \lceil \log n \rceil} (\sqrt{\log (1 / \delta)} + \sqrt{\log(n/\delta)})
\end{eqnarray}
\end{remark}

\section{Uniform Spanning Tree Basis}

The uniform spanning tree (UST) is a spanning tree generation technique that we will use to construct wavelet bases.
We will first examine the deep connection between electrical networks, USTs and random walks.
Because the UST is randomly generated, the test statistic, $\| \Bb_\Tcal \yb \|$ when conditioned on $\yb$ will also be random.
Due to results from cut sparsification, we can relate the performance of the UST wavelet detector to effective resistances.

\subsection{Cuts and Effective Resistance}

Effective resistances have been extensively studied in electrical network theory.
We define the combinatorial Laplacian of $G$ to be $\Delta = \nabla^\top \nabla$.
A {\em potential difference} is any $\zb \in \RR^E$ such that it satisfies {\em Kirchoff's potential law}: the total potential difference around any cycle is $0$.
Algebraically, this means that $\exists \xb \in \RR^V$ such that $\nabla \xb = \zb$.
The {\em Dirichlet Principle} states that any solution to the following program gives an absolute potential $\xb$ that satisfies Kirchoff's potential law:
\[
\textrm{min. } \xb^\top \Delta \xb \textrm{ s.t. } \xb_S = \vb_S
\]
for source/sinks $S \subset V$ and some voltage constraints $\vb_S \in \RR^S$.  
The realized objective $\xb^\top \Delta \xb$ is known as the {\em total energy} of the system.
By Lagrangian calculus, the solution to the above program is given by $\xb = \Delta^\dagger \vb$ where $\dagger$ indicates the Moore-Penrose pseudoinverse.
The effective resistance is the total energy of a system in which $v,w \in V$ are the source and sink respectively and a unit flow from $v$ to $w$ is induced.
Hence, the effective resistance between $v$ and $w$ is $r_{v,w} = (\delta_v - \delta_w)^\top \Delta^\dagger (\delta_v - \delta_w)$, where $\delta_v$ is the Dirac delta function.

A massively useful characterization of effective resistance is the random walk interpretation.
Let $X_t$ be the location of a random walker on $G$ at time $t$.
The hitting time $H(v,w)$ is then 
\[
H(v,w) = \EE [\min \{ t > 0 : X_t = w\} | X_0 = v]
\]
We find that the effective resistance is related to the hitting time by,
\[
r_{v,w} = \frac{H(v,w) + H(w,v)}{2m}
\]
The numerator is also known as the commute time.
As we will see, this characterization of effective resistance is useful when bounding it for specific graph models.

\subsection{UST Wavelet Detector}

In our framework, we are given the opportunity to evaluate our test according to our random algorithm.
We will now examine the performance of the spanning tree wavelet detector, when the spanning tree is drawn according to a UST.
First, we will explore the construction of the UST and examine key properties.
The UST is a random spanning tree, chosen uniformly at random from the set of all distinct spanning trees.
The foundational Matrix-Tree theorem \cite{kirchhoff1847ueber} describes the probability of an edge being included in the UST.
The following lemma can be found in \cite{lovasz1993random} and \cite{lyons2000probability}.

\begin{lemma}
\label{lem:mat_tree}
Let $G$ be a graph and $\Tcal$ a draw from $UST(G)$. 
\[
\PP \{ e \in \Tcal \} = r_e
\]
\end{lemma}

Hence, we can expect that for a given cut in the graph, that the cut size in the tree will look like the sum of edge effective resistances.
While it is infeasible to enumerate all spanning trees of a graph, the
Aldous-Broder algorithm is an efficient method for generating a draw from
$UST(G)$~\cite{aldous1990random}. The algorithm simulates
 a random walk on $G$, $\{X_t\}$, stops when all of the vertices
have been visited, and defines the spanning tree $\Tcal$ by the edges $\{(X_{H(X_0,v) - 1}, v) : v \in V\}$.

In order to control $\|\nabla_\Tcal \xb \|_0$, we need to control the overlap between a cut and the UST. 
Clearly the UST does not independently sample edges, but it does have the well documented property of negative association, that the inclusion of an edge decreases the probability that another edge is included.
The following lemma states a concentration result for the UST, based on negative association, and can be found in \cite{fung2010graph}.
The proof is a simple extension of the concentration results in \cite{gandhi2006dependent}.

\begin{lemma}
\label{lem:UST_conc}
Let $B \subset E$ be a fixed subset of edges, and $|\Tcal \cap B|$ denote the number of edges in $\Tcal$ also in $B$.
\[
\PP \{ |\Tcal \cap B| \ge (1 + \delta) \sum_{e \in B} r_e \} \le \left( \frac{e^\delta}{(1 + \delta)^{1 + \delta}} \right)^{\sum_{e \in B} r_e} 
\]
\end{lemma}

We use this result to give conditions under which the UST wavelet detector asymptotically distinguishes $H_0$ from $H_1$.

\begin{theorem}
\label{thm:Rperf}
Let $r_{\max} = \max_{\xb \in \Xcal} \sum_{e \in \supp(\nabla \xb)} r_e$ (the maximum effective resistance of a cut in $\Xcal$).
If
\[
\frac{\mu}{\sigma} = \omega \left( \sqrt{r_{\max} \log d} \log n \right)
\]
then $H_0$ and $H_1$ are asymptotically distinguished by the test statistic $\| \Bb \yb \|_\infty$ where $\Bb$ is the UST wavelet basis.
\end{theorem}

\begin{proof}
Let $r_B = \sum_{e \in B} r_e$ for $B \subset E$.
By some basic calculus, and the fact that $\log(1+x) \ge x / (1 + x/2)$, we see that
\[
\left( \frac{e^\delta}{(1 + \delta)^{1 + \delta}} \right)^{r_B} \le \exp \left(- \frac{\delta^2 r_B}{2 + \delta}  \right)
\]
Rewriting the Lemma \ref{lem:UST_conc}, we obtain with probability $>1-\gamma$
\[
|\Tcal \cap B| \le r_B + \sqrt{2 r_B \log \frac 1 \gamma + \frac 1 4 (\log \frac 1 \gamma)^2} + \frac{1}{2} \log \frac 1\gamma \le \left(r_B + \sqrt{2 r_B \log \frac 1 \gamma} + \log \frac 1\gamma \right) 
\]
Now, because $\|\nabla_\Tcal \xb \|_0 = |\Tcal \cap B|$ for $B = \supp(\nabla_\Tcal \xb)$, we know by Theorem \ref{thm:detection_bd} if
\[
\frac \mu \sigma = \omega \left( \sqrt{ \left(r_B + \sqrt{2 r_B \log \frac 1 \gamma} + \log \frac 1\gamma \right) \log d} \log n \right)
\]
then $H_0$ and $H_1$ are asymptotically distinguished and the result follows because we guarantee this uniformly for all such $B$.
\end{proof}


\section{Specific Graph Models}
\label{sec:examples}
In this section we study our detection problem for several different families of
graphs. Specifically, we control the effective resistance $r_e$for each graph
family, which when combined with Theorem~\ref{thm:Rperf} gives a lower bound
on the SNR for which $||{\bf By}||_\infty$ asymptotically distinguishes $H_0$
and $H_1$. 

In Theorem~\ref{thm:Rperf}, we showed that the consistency regime depends on the
effective resistances of the cuts induced by the class of signals
$\mathcal{X}$. On its own, it is not immediately clear that this result is an
improvement over the bound in Remark~\ref{rem:simple_bound} that we would obtain
from any spanning tree. However, Foster's theorem highlights why we expect the
effective resistance to be less than the cut size.

\begin{theorem}[Foster's Theorem \cite{foster1949average,tetali1991random}]
\[
\sum_{e \in E(G)} r_e = n - 1
\]
\end{theorem}

Hence, if we select an edge uniformly at random from the graph, we expect its
effective resistance to be $(n-1)/m \approx \bar{d}^{-1}$ (the reciprocal of the
average degree) where $m \triangleq |E(g)|$. Indeed, in several example graphs
we can formalize this intuition and give an improvement over
Remark~\ref{rem:simple_bound}.

We complement these results with two types of simulations verifying different
aspects of our theory. The first verifies the upper bound in
Lemma~\ref{lem:tree_cut_bound} for a variety of graph models by plotting $||{\bf
  B x}||_0$ versus $\rho \log(d) \log (n)$ for several randomly generated
signals. These plots (see Figure~\ref{fig:basis_sparsity}) demonstrate the
validity of our bound since in all cases $||{\bf Bx}||_0 \le \rho \log (d) \log
(n)$, but, more importantly, the readily-observable linear relationship between
these two quantities suggests that one should not expect an improvement on this
bound by more than a constant factor.

The second simulation verifies the performance of our spanning tree wavelets
detector on various graph models. In Figure~\ref{fig:power_curves}, we plot the
power of our test statistic (with Type I error fixed at 5\%) as a function of
signal strength $\mu$ for several values of $n$, where we allow $\rho$ to scale
with $n$ to ensure a non-empty $\mathcal{X}$. These simulations demonstrate that
as expected for sufficiently large signal strength, our statistic can separate
$H_0$ from $H_1$. More importantly, the threshold signal strength for which
detection is possible increases with $n$ and $\rho$, as predicted by our
theory.

\begin{figure}
\mbox{
\subfigure{\includegraphics[scale=0.23]{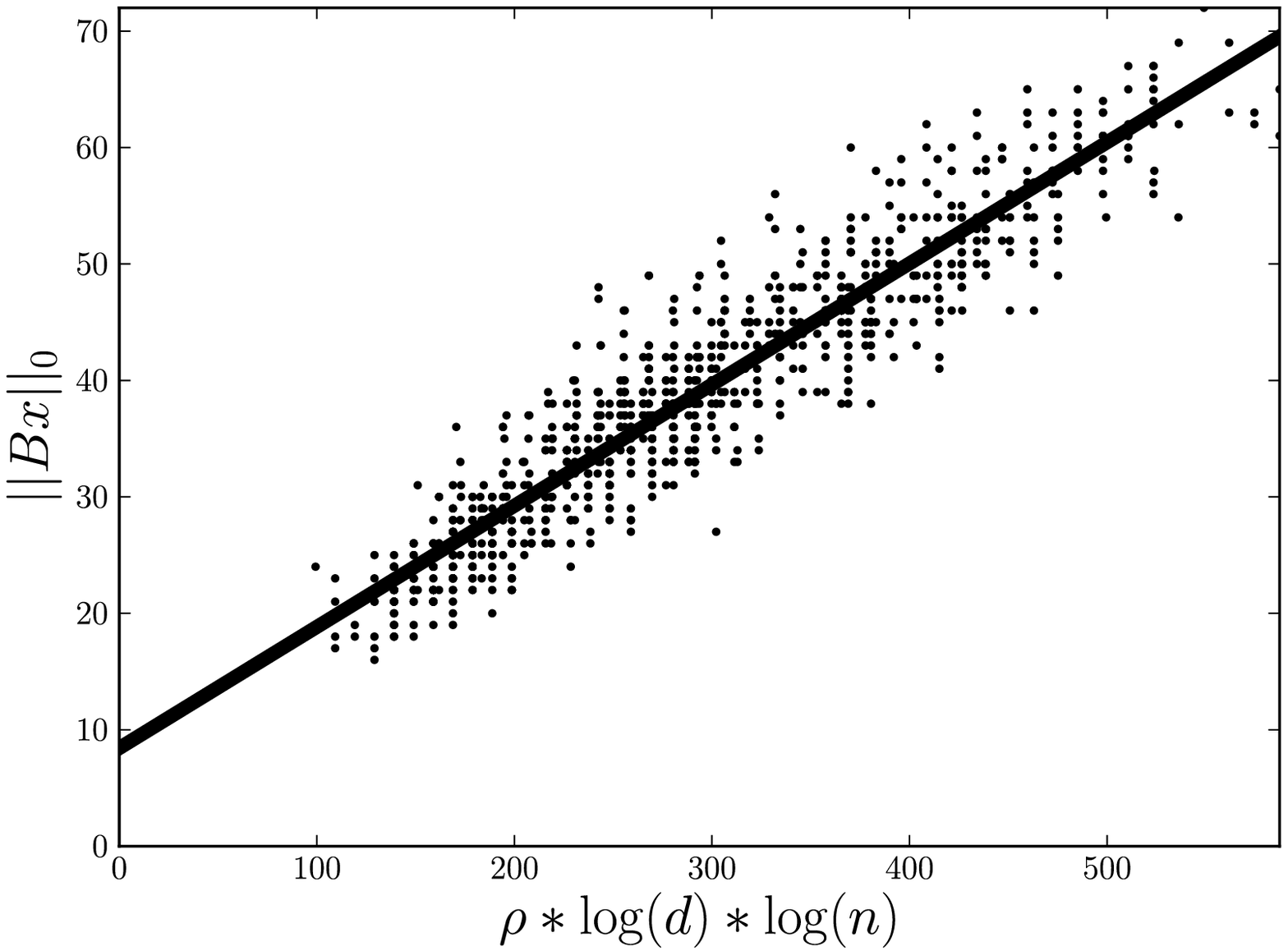}}\hspace{-0.4cm}
\subfigure{\includegraphics[scale=0.23]{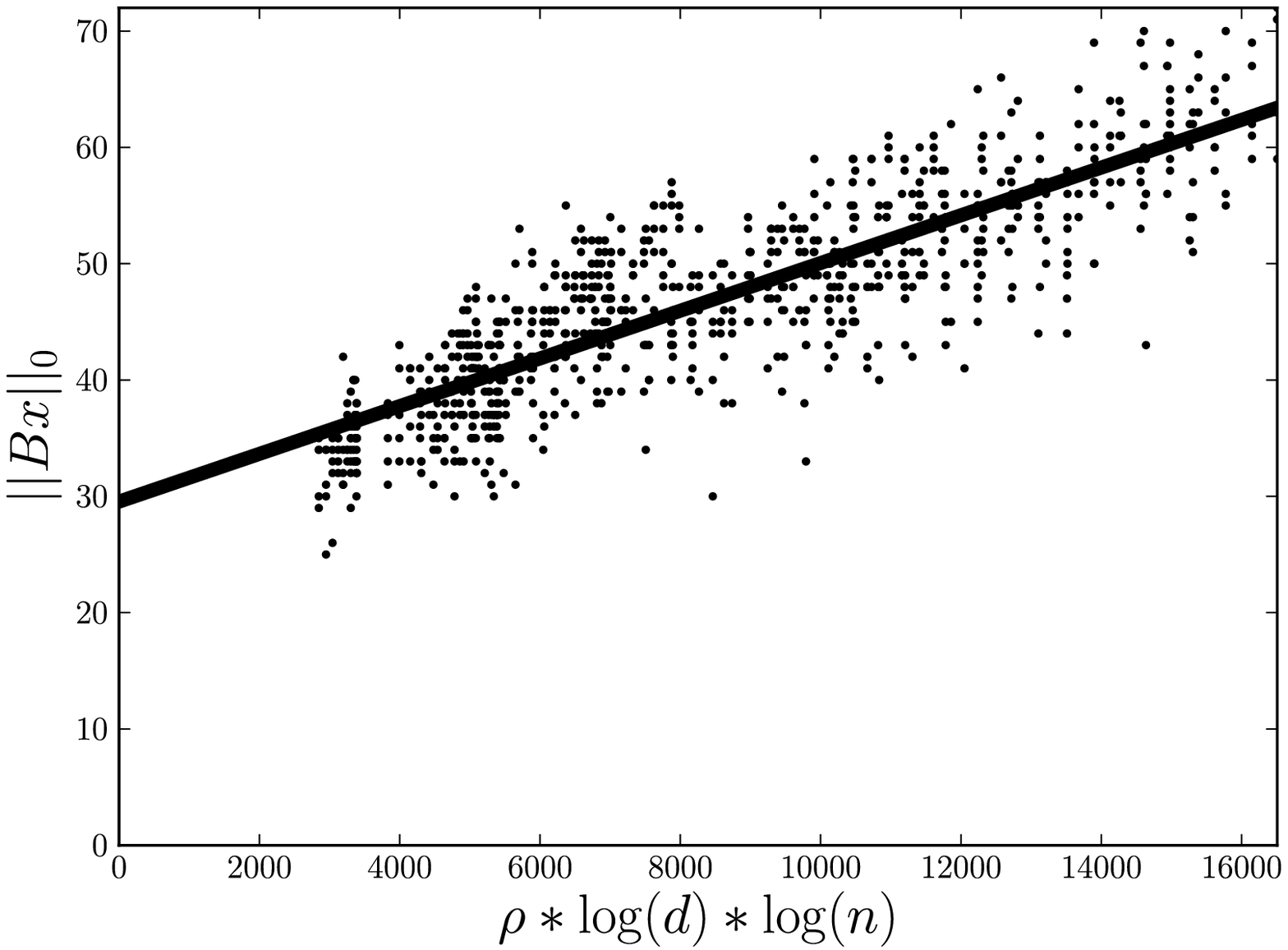}}\hspace{-0.4cm}
\subfigure{\includegraphics[scale=0.23]{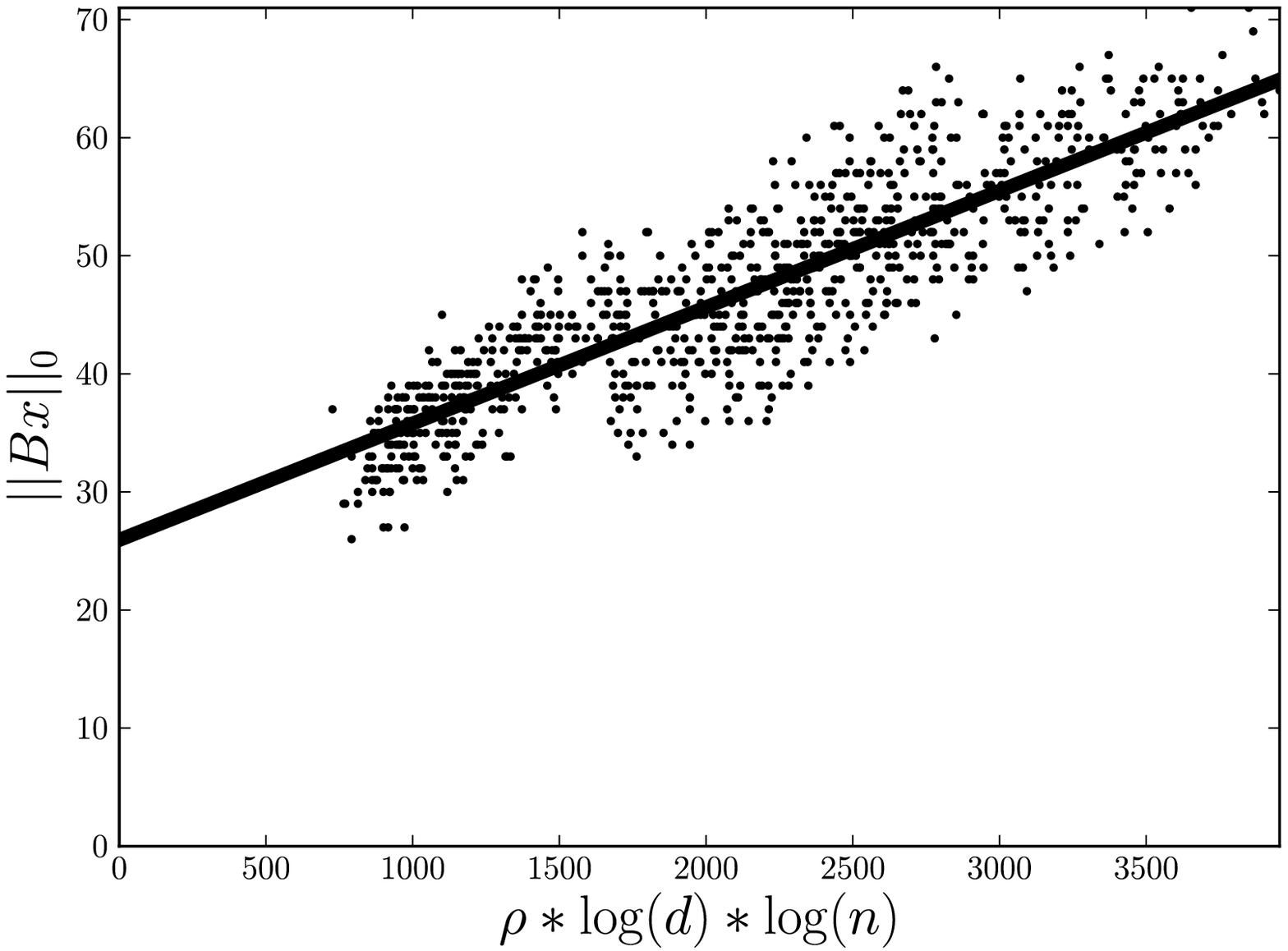}}\hspace{-0.4cm}
\subfigure{\includegraphics[scale=0.23]{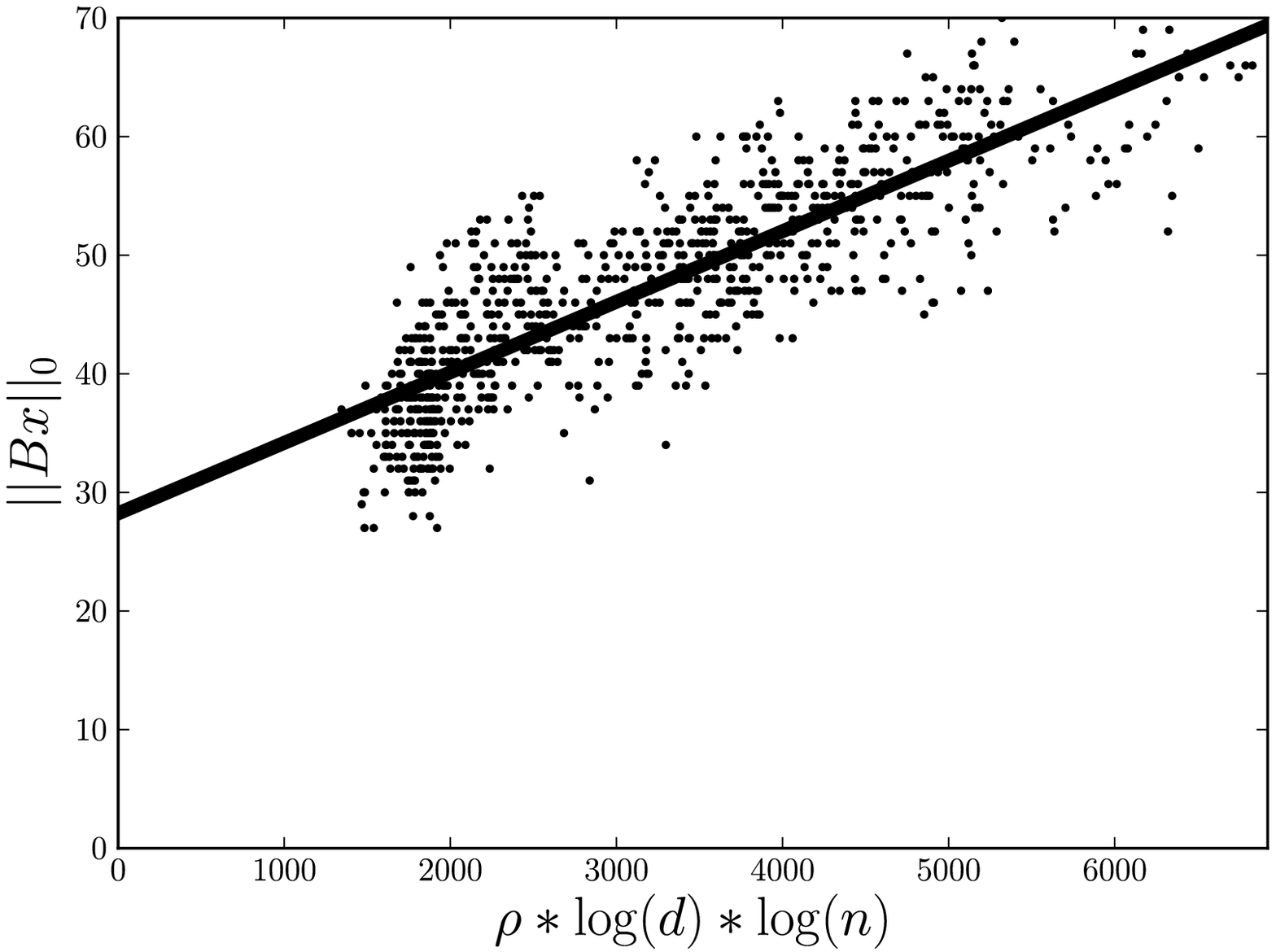}}
}
\caption{Spanning tree wavelet basis sparsity as a function of $\rho \log d \log
  n$ for, from left to right, 2-dimensional torus, complete, $k$-NN, and
  $\epsilon$ graphs. Linear fits have slopes: $0.10$, $0.0021$, $0.010$,
  $0.0059$ and $R^2$ coefficients: $0.88$ $0.72$, $0.76$, $0.71$ respectively.}
\label{fig:basis_sparsity}
\end{figure}

\begin{figure}
\mbox{
\subfigure{\includegraphics[scale=0.23]{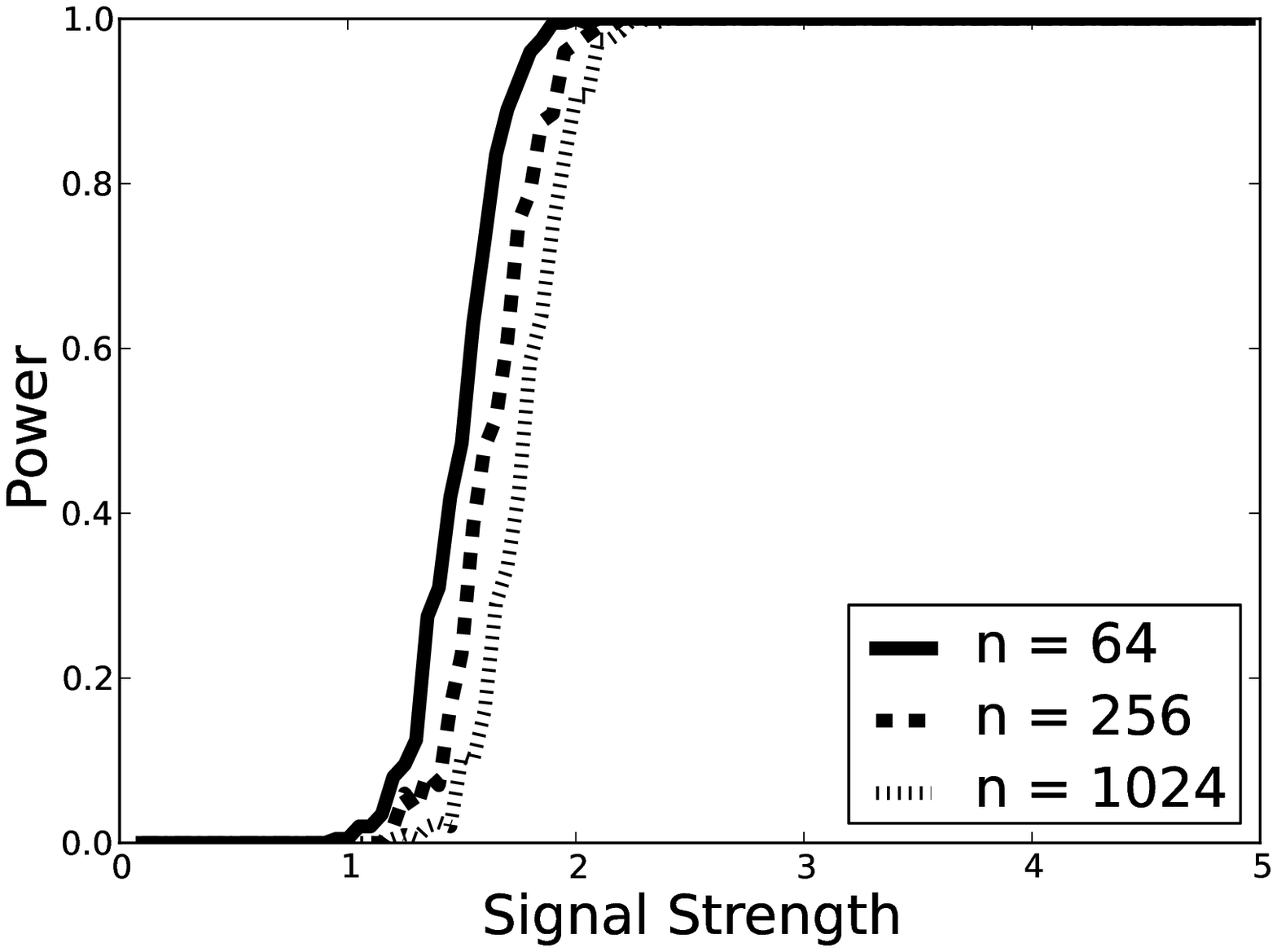}}\hspace{-0.4cm}
\subfigure{\includegraphics[scale=0.23]{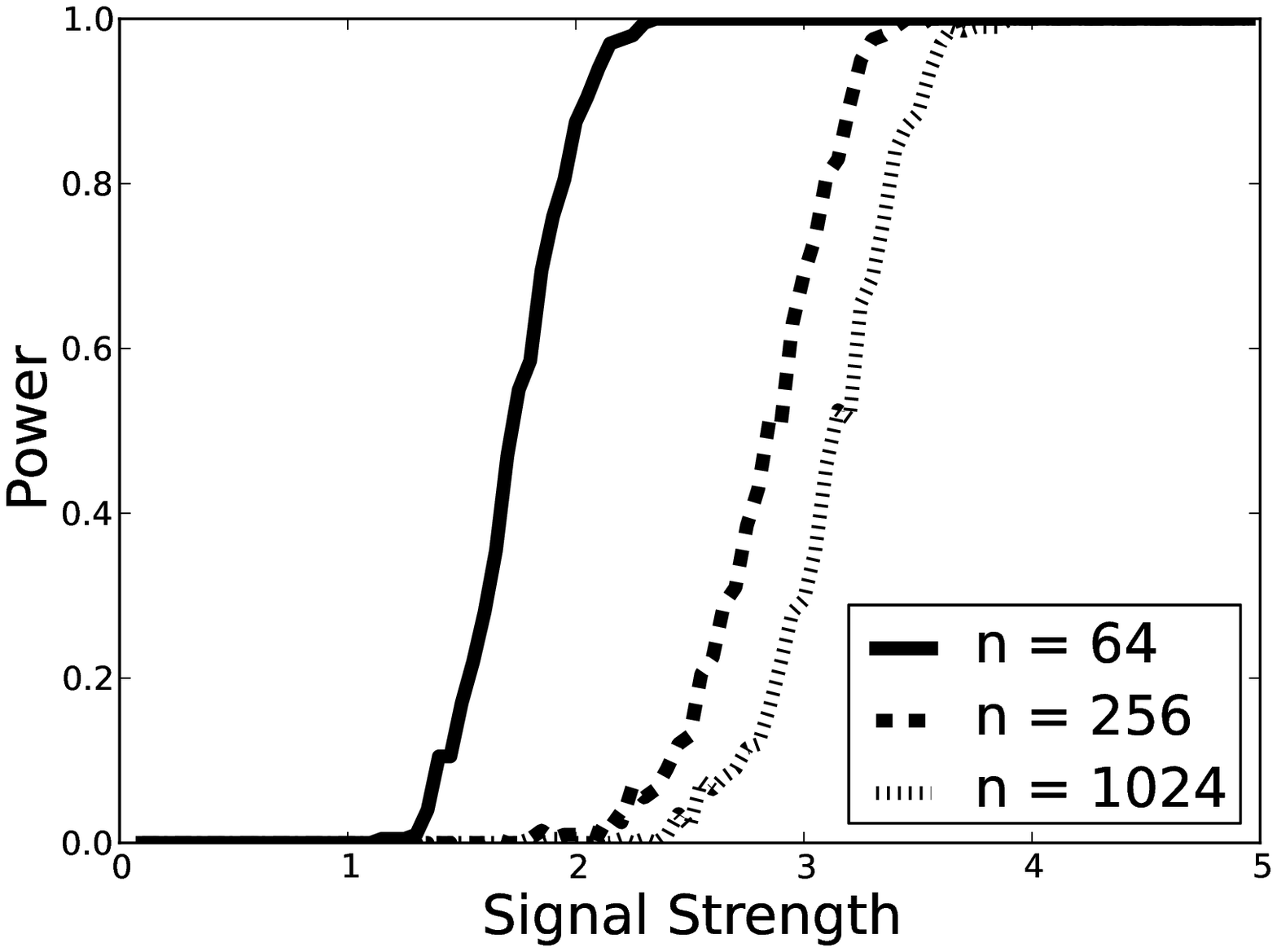}}\hspace{-0.4cm}
\subfigure{\includegraphics[scale=0.23]{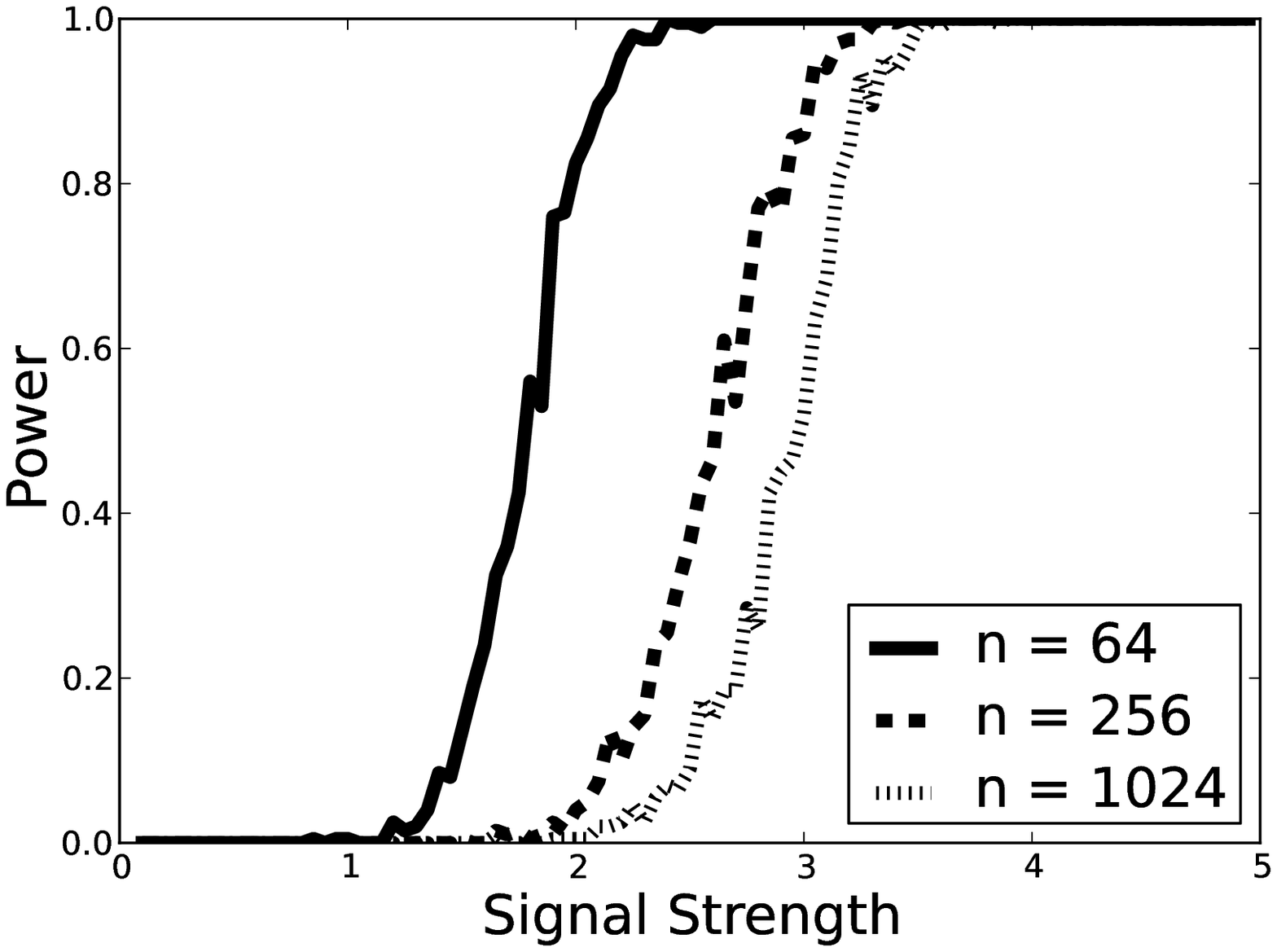}}\hspace{-0.4cm}
\subfigure{\includegraphics[scale=0.23]{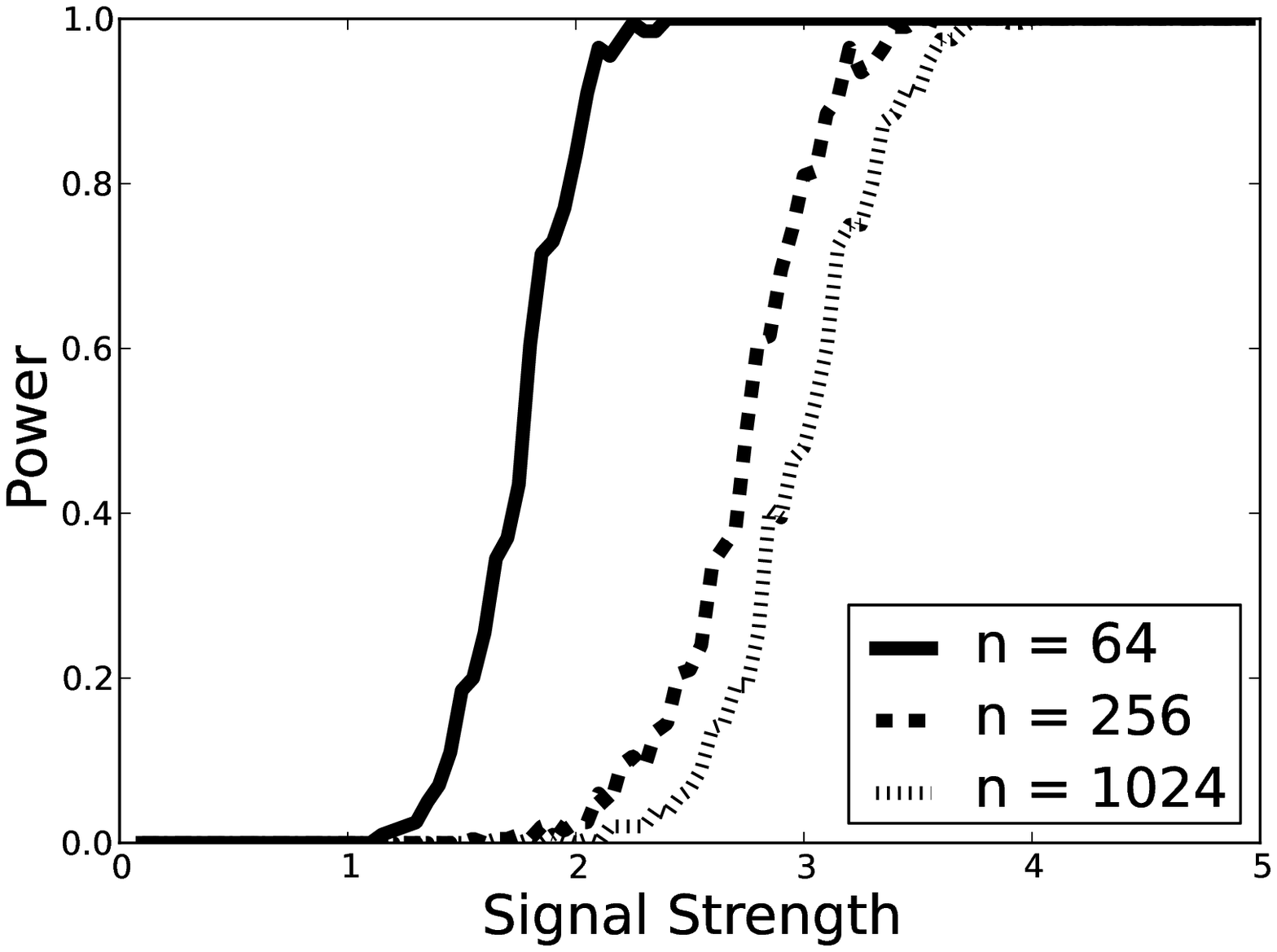}}
}
\caption{Power as a function of signal strength for different values of $n$ for
  2-dimensional torus, complete, $k$-NN, and $\epsilon$ graphs. $\rho$ scales
  like $\sqrt{n}$, $n$, $n^{2/3}$ and $n^{4/5}$ respectively.}
\label{fig:power_curves}
\end{figure}



\subsection{Edge Transitive Graphs}
An edge transitive graph, $G$, is one such that for any edges $e_0, e_1$, there
is a graph automorphism that maps $e_0$ to $e_1$. Examples of edge transitive
graphs include the $l$-dimensional torus and the complete graph $K_n$. For such
a graph, every edge has the same effective resistance, and Foster's Theorem then
shows that $r_e = (n-1)/m$ where $m$ is the number of edges. Moreover since edge
transitive graphs must be $d$-regular for some degree $d$, we see that $m =
\Theta(nd)$ so the $r_e = \Theta(1/d)$. This leads us to the following
corollary, which we note matches the lower bound in
Theorem~\ref{thm:lower_bound} modulo logarithmic terms if $\rho/d \le \sqrt{n}$:

\begin{corollary}
Let $G$ be edge transitive with common degree $d$.  Then for each edge $e \in E(G)$, $r_e = (n-1)/m$.
Consider the hypothesis testing problem
\eqref{eq:main_problem} where the set $\mathcal{X}$ is parameterized by
$\rho$. If: \[
\frac \mu \sigma = \omega \left( \sqrt{\frac \rho d \log d} \log n \right)
\]
Then the UST wavelet detector, $||{\bf By}||_{\infty}$, asymptotically
distinguishes $H_0$ and $H_1$.
\label{cor:edge_transitive_bd}
\end{corollary}

\subsection{kNN Graphs}
Oftentimes in applications, the graph topology is derived from data.  In this
case, the randomness of the data means that the graph itself is inherently
random.  Commonly, these graphs are modeled as random geometric graphs, and in
this section we will devote our attention to the {\em symmetric $k$-nearest
  neighbor graphs.}  Specifically, suppose that $\zb_1,...,\zb_n$ are drawn
i.i.d.~from a density $p$ supported over $\RR^d$.  Then we form the graph $G$
over $[n]$ by connecting vertices $i,j$ if $\zb_i$ is amongst the $k$-nearest
neighbors of $\zb_j$ or vice versa.  Some regularity conditions of $p$ are
needed for our results to hold; they can be found in \cite{vonluxburg2010}.

To bound the effective resistance $r_e$, Corollary 9 in \cite{vonluxburg2010}
shows that $H_{ij}/2m \rightarrow 1/d_j$ and by the definition of $r_e$ we see
that $r_{ij} \rightarrow \frac{1}{d_i} + \frac{1}{d_j} \le \frac{2}{k}$, since
$d_i \ge k$ for each $i$. A formal analysis leads to the following corollary,
which we prove in Appendix~\ref{app:examples} with more precise concentration arguments:
\begin{corollary}
Let $G$ be a $k$-NN graph with $k/n \rightarrow 0$ and $k(k/n)^{2/d} \rightarrow
\infty$ and where the density $p$ satisfies the regularity conditions
in~\cite{vonluxburg2010}. Consider the hypothesis testing problem
\eqref{eq:main_problem} where the set $\mathcal{X}$ is parameterized by
$\rho$. If:
\[
\frac{\mu}{\sigma} = \omega(\sqrt{\rho/k \log d} \log n)
\]
Then the UST wavelet detector, $||{\bf By}||_\infty$, asymptotically distinguishes $H_0$ and $H_1$.
\label{cor:knn_rate}
\end{corollary}

\subsection{$\epsilon$-Graphs}
The $\epsilon$-graph is another widely used random geometric graph in machine
learning and statistics. As with the $k$-NN graph, the vertices are embedded
into $\RR^d$ and edges are added between pairs of vertices that are within
distance $\epsilon$ of each other. As with the $k$-NN graph, Corollary 8
from~\cite{vonluxburg2010} shows that $H_{ij} \rightarrow m/d_j$ for each pair
of vertices. This leads us to believe that $r_{ij} \rightarrow1/(d_i) +
1/(d_j)$. If the density $p$ from which we draw data points is bounded from
below by some constant, then we can uniformly lower bound all of the degrees
$d_i$ using fairly elementary concentration results, which results in an upper
bound on $r_{e}$. Formalizing this intuition, we have the following corollary,
which we prove in Appendix~\ref{app:examples}:

\begin{corollary}
Let $G$ be a $\epsilon$-graph with points $X_1, \ldots X_n$ drawn from a density
$p$, which satisfies the regularity conditions in~\cite{vonluxburg2010} and is
lower bounded by some constant $p_{\min}$ (independent of $n$). Let $\epsilon
\rightarrow 0, n\epsilon^{d+2} \rightarrow \infty$ and consider the hypothesis
testing problem~\eqref{eq:main_problem} where the set $\mathcal{X}$ is parameterized by
$\rho$. If:
\[
\frac{\mu}{\sigma} = \omega(\sqrt{\frac{\rho}{n\epsilon^{d}} \log d} \log n)
\]
Then $||{\bf By}||_\infty$ asymptotically distinguishes $H_0$ and $H_1$.
\label{cor:epsilon_rate}
\end{corollary}

\section{Discussion}

We studied the detection of piece-wise constant activation patterns over graphs, and provided a necessary condition for the asymptotic distinguishability of signals that are assumed to have few discontinuities.
We gave a novel spanning tree wavelet construction, that is the extension of the Haar wavelet basis, for arbitrary graphs.
While it achieves strong theoretical performance for detection, the spanning tree wavelet construction could be of independent interest for compression and denoising.
The uniform spanning tree wavelet detector was shown to have strong theoretical guarantees that in many cases gives us near optimal performance.
This means that under adversarial choice of signal, our randomized algorithm asymptotically distinguishes $H_0$ from $H_1$.
Alternatively, this means that for any given signal (non-adversarial setting) that the vast majority of spanning trees induce detectors that asymptotically distinguish $H_0$ from $H_1$ for low signal-to-noise ratios.


\section*{Acknowledgements} This research is supported in part by AFOSR under grant FA9550-10-1-0382 and NSF under grant IIS-1116458.
AK is supported in part by an NSF Graduate Research Fellowship.


\bibliographystyle{plainnat}
\bibliography{biblio}
\vfill

\pagebreak
\appendix
\section{Proofs for Section~\ref{sec:wavelets}}
\subsection{Proof of Lemma~\ref{lem:tree_cut_bound}}
Before we proceed with the proof, we state and prove two results on the
performance of the algorithm:
\begin{lemma}
\label{lem:findbalance}
Let $\Tcal$ be a tree.  {\em FindBalance} returns a vertex $v$ such that the
largest connected component of $\Tcal \backslash v$ is of size at most $\lceil
|\Tcal|/2 \rceil$ in $O(|\Tcal|)$ time.
\end{lemma}

\begin{proof}
Let the objective be the size of the largest connected components of $\Tcal
\backslash v$.  Every move in {\em FindBalance} reduces the objective by at
least $1$ and the objective can be at most $|\Tcal| - 1$ so it must terminate in
less than $|\Tcal|$ moves.  Now at any step of {\em FindBalance}, if the
objective is greater than $\lceil |\Tcal|/2 \rceil$, the cumulative size of the
remaining connected components is less than $\lfloor |\Tcal|/2 \rfloor$.  Hence,
in the next step the connected component formed by these is less than $\lceil
|\Tcal|/2 \rceil$.  Thus, the program cannot terminate at a move directly after
the objective is greater than $\lceil |\Tcal|/2 \rceil$. 
\end{proof}

We will also require the following claim.  Indeed, controlling the depth of the
recursion in the wavelet construction is the sine qua non for controlling the
sparsity, $\| \Bb x \|_0$.
\begin{claim}
The wavelet construction has recursion depth at most $\lceil \log d \rceil
\lceil \log n \rceil$.
\end{claim}

\begin{proof}
Whenever {\em FormWavelet} is applied it increases the number of activated height of the dendrogram by at most $\lceil \log d \rceil$.
By lemma \ref{lem:findbalance} the size of the remaining components is halved, so the algorithm terminates in at most $\lceil \log n \rceil$ steps.
\end{proof}

\begin{proof}[Proof of Lemma~\ref{lem:tree_cut_bound}]
We will show that any edge $e \in \Tcal$ is activated by at most $\lceil \log d
\rceil \lceil \log n \rceil$ basis elements in $\Bb$, and this will imply the
result. We will say that an edge $e$ is activated by a basis element $\bb$ if $e
\subseteq supp(\nabla_\Tcal \bb)$.  It follows that for a basis element $\bb$,
if $\bb^T\xb \ne 0$ then $\exists e$ that is activated by $b$. Let
$\textrm{activations}(e)$ be the number of basis elements that activates $e$
($\textrm{activations}(e) = 0$ if $e \notin supp(\nabla_{\Tcal} \xb)$). We then
have
\[
||\Bb \xb||_0 \le \sum_{e \in supp(\nabla_{\Tcal} \xb)} \textrm{activations}(e)
\]

Consider some edge $e$. If $e$ is activated by some subtree $\Tcal_{sub}$ (we
use this interchangeably with being activated by the basis element formed by
partitioning $\Tcal_{sub}$ into two groups), then it can be activated by at most
one of $\Tcal_{sub}$'s subtrees. This implies that $\textrm{activations}(e)$ is
upper bounded by the depth of the recursion.
By the claim, we find that,
\[
||\Bb \xb ||_0 \le \sum_{e \in supp(\nabla_\Tcal \xb)} \lceil \log d \rceil \lceil \log n \rceil \le ||\nabla_\Tcal \xb||_0  \lceil \log d \rceil \lceil \log n \rceil
\]
Proving the first claim.
The second claim is obvious from the fact that $\Tcal$ contains a subset of the
edges in $\Gcal$, so every cut has larger cut size in $\Gcal$ than it does in
$\Tcal$.
\end{proof}

\subsection{Proof of Theorem~\ref{thm:detection_bd}}

\begin{proof}
Under the null $\xb = 0$, and we have that 
\[
\norm{\Bb \yb}_\infty = \norm{\Bb \epsilonb}_\infty < \sigma \sqrt{2 \log (n/ \delta)}
\]
with probability at least $1 - \delta$.
So, as long as $\tau = \sigma \sqrt{2 \log (n/ \delta)}$ then we control the probability of false alarm (type 1 error).
For a element $\xb$ of the alternative, let the index, $i^*$, achieve the maximum of $\Bb \xb$ (i.e.~$\norm{\Bb \xb}_\infty = |\Bb \xb|_{i^*}$).
Then $|\Bb \yb|_{i^*} \ge |\Bb x|_{i^*} - \sigma \sqrt{2 \log (1/\delta)}$  with probability at least $1 - \delta$ and
\[
|\Bb x|_{i^*}^2  = \norm{\Bb \xb}^2_\infty \ge \frac{\sum_{i : (\Bb \xb)_i \ne 0} (\Bb \xb)_i^2 }{\norm{\Bb \xb}_0 } = \frac{\norm{\xb}_2^2}{\norm{\Bb \xb}_0}
\]
Taking square roots and combining this with Lemma~\ref{lem:tree_cut_bound}, 
\[
\norm{\Bb \xb}_\infty \ge \frac{\norm{\xb}_2}{\sqrt{\norm{\nabla_\Tcal \xb}_0 \lceil \log d \rceil \lceil \log n \rceil}}
\]
from which we have the result,
\[
\norm{\Bb \yb}_\infty \ge \frac{\norm{\xb}_2}{ \sqrt{\norm{\nabla_\Tcal \xb}_0 \lceil \log d \rceil \lceil \log n \rceil} } -  \sigma \sqrt{2 \log (1/\delta)}
\]
Forcing this lower bound to be greater than $\tau$ gives us our result.
\end{proof}

\section{Proofs For Section~\ref{sec:examples}}
\label{app:examples}
\subsection{Proof of Corollary~\ref{cor:knn_rate}}
First we restate Corollary 9 from~\cite{vonluxburg2010}:
\begin{corollary}
Consider an unweighted symmetric or mutual $k$-NN graph built from a sequence
$X_1, \ldots, X_n$ drawn i.i.d. from a density $p$. Then there exists constants
$c_1, c_2, c_3$ such that with probability at least $1-c_1n\exp(-kc_2)$ we have
uniformly for all $i \ne j$ that:
\[
\left| \frac{k}{2m}H_{ij} - \frac{k}{d_j} \right| \le c_3 \frac{n^{2/d}}{k^{1+2/d}}
\]
\label{cor:luxburg_knn}
\end{corollary}

\begin{proof}[Proof of Corollary~\ref{cor:knn_rate}]
We focus on the symmetric $k$-NN graph in which we connect $v_i$ to $v_j$ if
$v_i$ is in the $k$-nearest neighbors of $v_j$ or vice versa. In this graph,
every node has degree $\ge k$ which will be crucial in our analysis. Our goal is
to bound the effective resistance of every edge, so that we can subsequently
bound $r_{max}$ and apply Corollary~\ref{thm:Rperf}. From the definition of
$r_e$ we have:
\begin{eqnarray*}
r_{ij} &=& \frac{1}{2}\left(\frac{H_{ij}}{m} + \frac{H_{ji}}{m}\right)\\
&\le& 2c_3 \frac{n^{2/d}}{k^{2+2/d}} + \frac{1}{d_i} + \frac{1}{d_j}\\
&\le&2c_3 \frac{n^{2/d}}{k^{2+2/d}} + \frac{2}{k}
\end{eqnarray*}
Where the first line is the definition of $r_{ij}$, the second line follows from
Corollary~\ref{cor:luxburg_knn} and the last line follows from the fact that
$d_i \ge k$ for each vertex. Since $k(k/n)^{2/d} \rightarrow \infty$, we see
that $r_{ij} = O(\frac{1}{k})$. Moreover, with this scaling of $k$, that the
probability in Corollary~\ref{cor:luxburg_knn} is going to 1. We can therefore
bound $r_{max}$ as:
\begin{eqnarray*}
r_{max} \le \rho \left(2c_3 \frac{n^{2/d}}{k^{2+2/d}} + \frac{2}{k}\right) = O\left(\frac{\rho}{k}\right)
\end{eqnarray*}
Since the first term is going to zero with $n$. Plugging in this bound on
$r_{max}$ into Theorem~\ref{thm:Rperf} gives the result.
\end{proof}

\subsection{Proof of Corollary~\ref{cor:epsilon_rate}}
As before, we first state Corollary 8 from~\cite{vonluxburg2010}:
\begin{corollary}
Consider an unweighted $\epsilon$-graph built from the sequence $X_1, \ldots,
X_n$ drawn i.i.d. from the density $p$ .Then there exists constants $c_1, \ldots
c_5 > 0$ such that with probability at least $1 - c_1 n \exp(-c_2n\epsilon^d) -
c_3 \exp(-c_4n\epsilon^d)/\epsilon^d$, we have uniformly for all $i \ne j$ that:
\[
\left| \frac{n\epsilon^d}{2m} H_{ij} - \frac{n\epsilon^d}{d_j}\right| \le \frac{c_5}{n\epsilon^{d+1}}
\]
\label{cor:luxburg_epsilon}
\end{corollary}

\begin{proof}[Proof of Corollary~\ref{cor:epsilon_rate}]
Some manipulation of the result in Corollary~\ref{cor:luxburg_epsilon} reveals
that:
\[
H_{ij} \le \frac{2m}{d_j} + \frac{2c_5 m}{n^2 \epsilon^{2d+2}}
\]
Under our scaling, the second term goes to zero and the probability in
Corollary~\ref{cor:luxburg_epsilon} goes to one, so $H_{ij} = O(m/d_j)$. We will
now give a lower bound on $d_j$. If $X_i$ is in the ball of radius $\epsilon$
centered at $X_j$, then we connect $X_i$ and $X_j$. Thus $d_j$ is exactly the
number of vertices in the $B(X_j; \epsilon)$. The regularity condition on $p$
in~\cite{vonluxburg2010} requires that there exists constants $\alpha$ and
$\epsilon_0$ such that for all $\epsilon < \epsilon_0$ and for all $x \in
\textrm{supp}(p)$, $\textrm{vol}(B(x; \epsilon) \cap \textrm{supp}(p)) \ge
\alpha \textrm{vol}(B(x; \epsilon))$. By this fact, the fact that the density is
lower bounded by $p_{min}$, and by the fact that $\epsilon \rightarrow 0$, we
know that for sufficiently large $n$, $p(B(X_j; \epsilon)) \ge p_{\min} \alpha
c_d \epsilon^d$ where $c_d \epsilon^d$ is the volume of a $d$-dimensional ball
of radius $\epsilon$. The probability that $X_i \in B(X_j; \epsilon)$ is
distributed as a Bernoulli random variable with mean $\ge \alpha p_{min} c_d
\epsilon^d$. By Hoeffding's inequality and a union bound we get that:
\[
d_j \ge n\alpha p_{min} c_d\epsilon^d + \sqrt{n \log (n)} = \Omega(n \epsilon^d)
\]
for all vertices $j$ with probability $\ge 1-1/n$. Using the definition of
$r_{i,j}$ along with the bound on $H_{ij}$ and $d_j$ we have that uniformly for
all pairs $i,j$:
\[
r_{i,j} = O(\frac{1}{n\epsilon^d})
\]
Plugging in this bound into Theorem~\ref{thm:Rperf} gives us the result.
\end{proof}

\end{document}